\documentclass[letterpaper, 10 pt, conference]{ieeeconf}  

\IEEEoverridecommandlockouts                              %

\usepackage{color}
\usepackage{upgreek}

\usepackage[T1]{fontenc}    
\usepackage{hyperref}       
\usepackage{url}            
\usepackage{booktabs}       
\usepackage{amsfonts}       
\usepackage{nicefrac}       
\usepackage{microtype}      
\usepackage{amssymb,latexsym,amsfonts,amsmath}
\usepackage{graphicx}
\usepackage{color}
\usepackage{defs_main}
\usepackage{mathtools}

\newcommand{\map}[3]{#1:#2 \rightarrow #3}

\newcommand{\KL}{\mathrm{KL}}
\newcommand{\FR}{\mathrm{FR}}

\newcommand{\Diag}{\operatorname{diag}}

\newcommand{\minthetae}{\upscr{\theta_e}{min}}

\usepackage{tikz}
\usetikzlibrary{arrows.meta}
\usetikzlibrary{positioning}
\usetikzlibrary{shapes.geometric}
\usetikzlibrary{calc}

\begin{document}

\title{Preventing Model Collapse: A Fisher-Rao Perspective\\on the Dynamics of Training with Synthetic Data}

\author{Matteo Marchi$^{1}$, João Pedro Silvestre$^{1}$, Bahman Gharesifard$^{2}$, and Paulo Tabuada$^{1}$\thanks{$^{1}$ Matteo Marchi, João Pedro Silvestre, and Paulo Tabuada are with the Electrical and Computer Engineering Department, University of California at Los Angeles, Los Angeles, CA 90095 USA (e-mail: {\tt\small \{matmarchi,joaosilvestre,tabuada\}@ucla.edu}).}\thanks{$^{2}$Bahman~Gharesifard is with the Department of Mathematics and Statistics at Queen's University, Kingston, ON, Canada (e-mail: {\tt\small bahman.gharesifard@queensu.ca}).}\thanks{This research was supported in part by the US ARL Cooperative Agreement W911NF-17-2-0196 and by the NSF award 2502536. The work of João Pedro Silvestre was partially supported by the PhD fellowship 2023.01843.BD from the Fundação para a Ciência e a Tecnologia (FCT), Portugal.}
}

\maketitle
 \thispagestyle{empty}
\pagestyle{empty}

\begin{abstract}
    Large Language Models (LLMs) are now routinely trained using synthetic data, since high-quality human data has been exhausted by the ever increasing needs of larger and larger models.
    However, recursive training on synthetic data frequently induces model collapse, a degenerative feedback loop where models progressively forget the true underlying data distribution.
    Training on a mixture of synthetic and fresh human data is a logical countermeasure and can prevent model collapse. However, it is an open question as to what is the exact minimum required ratio of human-to-synthetic data to maintain training stability.

    In this paper, we establish rigorous theoretical guarantees on the minimum rate of human data required to prevent model collapse. Although previous work established a formal lower bound for this ratio, such bound can be vacuous for very high dimensions, as the analysis relies on the usual Euclidean metric in $\R^n$ and is not adapted to the space of categorical probability distributions.
    Instead, in this paper we explicitly leverage the information-geometric structure of the probability simplex by analyzing the dynamics of the process under the Fisher-Rao metric. We derive quantitative contraction and invariance bounds that are stable and do not become trivial as the dimensions increase. Thus, we show that the effective required data ratio to prevent model collapse is different
    than previously implied.
\end{abstract}

\section{Introduction}

In recent years, generative AI, and in particular Large Language Models (LLMs) have become deeply ingrained in our society, primarily driven by a remarkable leap in performance and generation capabilities of recent models \cite{zhao2023survey}. Today, modern LLMs can produce text that is virtually indistinguishable from human writing; in fact, recent studies show that human evaluators are often misguided by flawed heuristics when trying to identify AI-generated language \cite{jakesch2023human}.

However, this sudden leap in quality is not without drawbacks, including steep economic costs \cite{bender2021dangers}, heavy computational demands \cite{strubell2019energy}, and a reliance on increasingly vast training corpora \cite{hoffmann2022training}. Driven by the empirical scaling laws required to push state-of-the-art performance, developers are now training foundational models on trillions of tokens \cite{touvron2023llama}. Unfortunately, sustaining this path requires such a large amount of data that some studies project the supply of fresh, high-quality human text will soon be completely exhausted~\cite{villalobos2022will}.

The high caliber of AI-generated text presents a tempting solution to this impending data scarcity by leveraging the models' own synthetic outputs for future training. However, recursive training on machine-generated data leads to a critical failure mode commonly referred to as \emph{model collapse}, as demonstrated both empirically \cite{shumailov2024ai, alemohammad2023self} and theoretically \cite{marchi2024heat}. Instead of learning effectively, models caught in this degenerative feedback loop deviate from the true data distribution, amplifying their own errors and producing repetitive, homogenized outputs \cite{shumailov2024ai, herel2024collapse}. Integrating fresh human data into the iterative training cycle appears to be a logical countermeasure, but current evidence suggests that simplistic strategies, such as injecting small, fixed proportions of real data, are insufficient to stop model collapse~\cite{briesch2023large}. Alternative mitigation strategies are being actively investigated, such as employing data verification and curation pipelines to filter out degraded synthetic outputs~\cite{feng2024beyond}. Although curating synthetic data can delay the onset of degeneration, it introduces substantial computational overhead and relies heavily on the quality and robustness of the verifier itself.
Recent empirical and theoretical works have established that incorporating sufficient human data can prevent model collapse \cite{gerstgrasser2024model, gharesifard2025preventing}, yet characterizing the precise dynamics of this mixed-data regime remains a significant challenge. This naturally raises a fundamental question: can we rigorously bound the amount of human data required to preclude model collapse?

We tackle this problem by building upon the framework introduced in \cite{gharesifard2025preventing}, modeling the iterative training of generative models as a closed-loop stochastic process. This previous work has primarily focused on analyzing asymptotic equilibrium states \cite{marchi2024heat, gharesifard2025preventing} and worked with the traditional Euclidean metric, which progressively distorts distances \emph{between probability distributions} as the dimensions increase. Our approach fundamentally departs from this prior analysis by explicitly leveraging the information-geometric structure of the probability simplex. Specifically, by working with the Fisher-Rao metric, we derive quantitative contraction and invariance bounds that remain stable and meaningful as the underlying dimensions of the model increase. This geometric perspective shows that the effective amount of human data required to prevent collapse is greater than previously implied.

\section{Notation and Preliminaries}

\subsection{Notation}

We denote by $\R^n$ the $n$-dimensional Euclidean space, $\R^+_0$ as the set of nonnegative real numbers, $\mathbb N$ as the set of natural numbers with zero, \mbox{$\Delta^n \triangleq \left\{x \in (\R_{0}^+)^n ~|~ \sum_{i=1}^n x_i = 1\right\}$} as the $n$-dimensional probability simplex, $\Vert\cdot\Vert_1$ as the 1-norm, $\Vert\cdot\Vert_2$ as the 2-norm, and $\langle\cdot,\cdot\rangle$ as the inner product.

We use standard asymptotic notations $O(\cdot)$ and $\Theta(\cdot)$ to describe the limiting behavior of sequences. In particular, for sequences of functions $\map{a_n}{\mathbb N}{\real} $ and $\map{b_n}{\mathbb N}{\real_{>0}} $, we write $a_n = O(b_n)$ whenever  there exist $c\in\real^+$ and $n_0\in\mathbb N$ such that:
\[
|a_n| \le c\, b_n, \qquad\text{for all } n\ge n_0.
\]
We also write $a_n = \Theta(b_n)$ whenever there exist $c_1,c_2\in\real^+0$ and $n_0\in\mathbb N$ such that:
\[
c_1\, b_n  \le  |a_n|  \le  c_2\, b_n,
\qquad\text{for all } n\ge n_0 .
\]
Given vectors $ v\in [0,1]^n $ and $ w \in ]0,1]^n$, we define: 
\[
\|v\|_{\Diag(w)}^2=v^\top\Diag(w)v,
\]
where $\Diag(w)\in\R^{n\times n}$ denotes the square matrix whose diagonal consists of the entries of $w$, and whose off-diagonal elements are $0$.
Given a scalar function ${f:\real\to\real}$ and a vector argument $x\in\real^n$, we denote ${f(x) = \left(f(x_1), f(x_2),\dots,f(x_n)\right)\in\R^n}$ as its element-wise application to $x$. Consider the interior of the probability simplex $\Delta^n_{\mathrm{int}}=\Delta^n \backslash \partial \Delta^n$. We equip $\Delta^n_{\mathrm{int}}$ with the \emph{Fisher-Rao} Riemannian metric $g_\theta(u,v)=\sum_{i=1}^n \frac{u_i v_i}{\theta_i}$, for $u,v$ in the tangent space $T_\theta \Delta^n = \{ u \in \mathbb{R}^n : \sum_{i=1}^n u_i = 0 \}$. This metric induces the geodesic Hellinger distance on the probability simplex~\cite{miyamoto2024closed}:
\begin{equation}\label{eq:Hellinger}
d_{\FR}(\theta,\vartheta)
=\arccos \Big(\sum_{i=1}^n \sqrt{\theta_i\vartheta_i}\Big),
\end{equation}
which is naturally related to the squared Hellinger distance $H^2(\theta,\vartheta) := \frac{1}{2}\|\sqrt{\theta}-\sqrt{\vartheta}\|_2^2 = 1 -  \langle\sqrt{\theta}, \sqrt{\vartheta}\rangle$.

\begin{definition}[Kullback-Leibler divergence]
Let $\theta,\vartheta \in \Delta^n_{\mathrm{int}} $.
The \emph{Kullback-Leibler (KL) divergence} of $\theta$ from $\vartheta$ is:
\[
D_{\mathrm{KL}}(\theta\|\vartheta)
:=\sum_{i=1}^n \theta_i \log\frac{\theta_i}{\vartheta_i}.
\]
\end{definition}
$D_{\mathrm{KL}}$ is non-negative and equals zero if and only if $\theta=\vartheta$.

\subsection{Generative Models}

In this section we describe the mathematical model used to analyze a generative model and its iterative training process. This is largely based on the model first presented in~\cite{marchi2024heat} that we extend in this work.

We define a generative model to be a function $\phi:\mathbb{R}^p\to\Delta^n$ that maps a parameter vector $w\in\R^p$ to an output distribution $\Theta=\phi(w)\in\Delta^n$. The  $i$-th entry of $\phi(w)$ is the \emph{nominal} probability of producing the $i$-th element from a list of outcomes $\overline{\outcome} = \{\outcome_{1}, \dots \outcome_{\varn}\}$ when the model is queried. With no loss of generality, we assume that the $i$-th element of $\overline{\mathcal Y}$ is the $n$-dimensional vector containing $1$ in its $i$-th entry and $0$ in all others\footnote{Within machine learning literature, this is referred to as ``one-hot encoding'' of a categorical variable.}.

In practice, when a model generates data, the \emph{actual} output probability distribution is modulated via a temperature function $\tau:\Delta^{n}\rightarrow\Delta^{n}$, defined for $i=1,2,\hdots,n$ as:
\begin{equation}\label{eq:temp-motivation}
\tau_i(\Theta)=
\frac{\Theta_i^{1/T}}{\sum_{j=1}^n \Theta_j^{1/T}},
\qquad T>0.
\end{equation}
The temperature function describes the common practice of converting a $\frac{1}{T}$-scaled vector of raw logits, produced by a generative model, into a vector of probabilities. This specific form of $\tau$, induced by the standard \emph{softmax} function, is ubiquitous across nearly all modern generative models used in practice. For a more detailed discussion and a broader class of temperature functions, we refer the reader to~\cite{marchi2024heat}.

\subsection{Iterative Training}

We now consider a sequence of generative models, ${\Theta(k) = \phi(w(k))}$, indexed by $k\in\mathbb N$, each trained on a dataset $\mathcal D_k$ with cardinality $\ell_k = |\mathcal{D}_k|$. The dataset at any time step $k$ is a multiset\footnote{Multisets generalize sets by allowing them to contain multiple instances of the same element. We cannot use a simple set to describe the training data, as we need to track the relative frequency of each element in $\mathcal D_k$.} $\mathcal D_k = \{Y_{1}, \ldots, Y_{\ell_k}\}$ where each element belongs to the set of possible outcomes $\overline{\mathcal Y}$. To any non-empty dataset, we can associate a corresponding ``empirical'' probability vector:
\begin{equation}
    \Theta'(k) = \frac{1}{\ell_{k}}\sum_{i=1}^{\ell_{k}}Y_i,
\end{equation}
whose entries are the relative frequencies of each possible outcome within the dataset.
$\mathcal D_k$ evolves by accumulating some amount of ``fresh'' human generated data\footnote{While we use the expression ``human data'' to evoke the notion of human-produced content on the Internet, this refers to any kind of data coming from a fixed external probability distribution.} and some amount of synthetic data generated by the model trained at the current time step. Specifically, at any time step $k$, we assume that $\alpha_{k}\in \mathbb{N}$ outcomes of $\overline{\mathcal Y}$ are sampled according to $\tau(\Theta(k))$ to form:
$$\mathcal{D}^{\mathrm{syn}}_k=\{Y_{\ell_k+1}, \ldots, Y_{\ell_k+\alpha_{k}}\}.$$
Additionally, $\beta_{k}\in \mathbb{N}$ outcomes of $\overline{\mathcal Y}$ are sampled according to a fixed external distribution $\humD\in\Delta^n$ to form:
$$\mathcal{D}^{\mathrm{human}}_k=\{Y_{\ell_k+\alpha_{k}+1}, \ldots, Y_{\ell_k+\alpha_{k}+\beta_{k}}\}.$$
The training dataset available at step $k+1$ is thus:
$${\mathcal{D}_{k+1} = \mathcal{D}_k \cup \mathcal{D}^{\mathrm{syn}}_k \cup \mathcal{D}^{\mathrm{human}}_k},$$ of cardinality $\ell_{k+1}=\ell_k+\alpha_{k}+\beta_{k}$.

 The training of a generative model at time step $k+1$ occurs by using the newly available dataset $\mathcal{D}_{k+1}$ (and possibly the previously trained $w(k)$) to compute a new parameter vector $w(k+1) = f(w(k),\mathcal{D}_{k+1})$, where $f$ abstracts away the details of the training and optimization process.
 Then, the evolution of $\Theta$ is described by the stochastic process:
\begin{equation}\label{eq:stochastic-motivation}
\Theta(k+1)
=\phi\big(f(w(k),\mathcal{D}_{k+1})\big).
\end{equation}

In~\cite{marchi2024heat}, the authors analyze the asymptotic behavior of~\eqref{eq:stochastic-motivation} in the absence of fresh data ($\beta_k = 0$). They show that in the limit of $k\to\infty$ the distribution learned by the generative model exhibits a high degree of degeneration or model collapse with high probability. Specifically, $\Theta(k)$ becomes arbitrarily close to the boundary of the simplex (in fact, to the corners of the simplex) or to its center (uniform probability distribution).
The question of whether this degeneration can be mitigated by injecting human data in the loop was investigated by the authors of~\cite{gharesifard2025preventing}.
Assuming $\beta_k =\mu \alpha_k$ for a constant ratio $\mu >0$, the limiting behavior of \eqref{eq:stochastic-motivation} is determined almost surely by the behavior of the continuous-time dynamical system:
\begin{equation}\label{eq:cts}
    \dot{\theta}(t) = \tau(\theta(t)) - \theta(t) + \mu\big(\theta_0-\theta(t)\big) + \varepsilon(t),
\end{equation}
where $\theta_0 := \humD$ is the human distribution, $\varepsilon(t)$ is a bounded perturbation, and we assume the flow of \eqref{eq:cts} preserves the simplex, implying $\sum_{i=1}^n\dot\theta_i=0$. The magnitude of the perturbation $\varepsilon$ is a measure of training accuracy, and a model that learns a dataset distribution with low error has a correspondingly small perturbation $\varepsilon$.

Under these assumptions, the results in~\cite{gharesifard2025preventing} state that for a sufficiently high ratio of human data $\mu$, the trajectories of \eqref{eq:cts} converge to an Euclidean ball $\mathbb{B}(\theta_e,\upepsilon)$ around an equilibrium $\theta_e$ satisfying $\tau(\theta_e) - \theta_e + \mu(\theta_0-\theta_e)=0$ and provide an expression bounding the size of this ball as a function of $\mu$ and the other problem parameters.
However, analyzing these dynamics under the Euclidean metric yields bounds that become increasingly uninformative as the dimension $n$ of the probability simplex grows. For $n\to\infty$, the Euclidean distance between almost any two probability distributions approaches zero~\cite{aggarwal2001surprising}. Thus, requiring trajectories to converge to an Euclidean ball of a fixed radius becomes a progressively weaker condition, as such a ball eventually encompasses the majority of the simplex regardless of the choice of ratio $\mu$. For an example illustrating this, take $n$ to be even and consider the family of ``disjoint'' probability vectors of form $p = \left(\frac{2}{n}, \dots, \frac{2}{n}, 0, \dots, 0\right)$ and $q = \left(0, \dots, 0, \frac{2}{n}, \dots, \frac{2}{n}\right)$. Despite representing completely distinct categorical outcomes, their Euclidean distance scales as $O(1/\sqrt{n})$ and vanishes as $n \to \infty$.
By contrast, the Fisher-Rao metric captures the underlying information-geometric structure of $\Delta^n$ (see~\cite{amari2016information}) and assigns a constant positive distance ${d_\FR(p,q) = \arccos(0) = \frac{\pi}{2}}$ between $p$ and $q$ regardless of $n$. Because the Euclidean metric progressively under-penalizes the distance between distributions, it yields an overly optimistic assessment of the required human data scaling. The main contribution of this paper is to substantially refine this analysis by working directly on the Fisher-Rao manifold of $\Delta^n$.

\section{Main Result}

We first need to make the following assumptions. We require that the magnitude of the perturbation $\varepsilon$ is bounded, the human data distribution does not lie exactly on the boundary of the probability simplex, and that there exists an equilibrium of~\eqref{eq:cts} when $\varepsilon=0$.

\begin{assumption}\label{ass:eps_theta0}
There exists $\eta\ge 0$ such the term $\varepsilon(t)\in\R^n$ in~\eqref{eq:cts} satisfies $\Vert{\varepsilon(t)}\Vert_{\infty}\le \eta$ for all $t\ge 0$, and all entries of $\theta_0$, the human distribution, are strictly positive: $$\delta:=\min_i \theta_{0,i}>0.$$
\end{assumption}

\begin{assumption}\label{ass:theta_e}
There exists $\theta_e\in\Delta^n$ satisfying:
\begin{equation}\label{eq:eq}
    \tau(\theta_e) - \theta_e + \mu(\theta_0-\theta_e)=0.
\end{equation}
\end{assumption}

Further, it is convenient to define the following quantity:
\begin{definition}\label{def:eta_FR}
Let $\minthetae$ be the smallest element of $\theta_e$, then we define the normalized inverse temperature $\eta_\FR$ as:
\[
\eta_{\FR}\ :=\ \frac{1}{T\,\minthetae}.
\]
\end{definition}

We can now introduce the main contribution of this work in the following theorem. This result identifies a ball in the Fisher-Rao metric that $\theta$ will converge to, the convergence rate to this ball, and a minimum threshold for the human-to-synthetic data ratio that guarantees this behavior. Here, we merely state the theorem and prove it in the following section.
\begin{theorem}\label{thm:main}
Suppose that Assumptions~\ref{ass:eps_theta0}-\ref{ass:theta_e} hold and that $\mu\delta>\eta$, where $\mu$ is the ratio between human and synthetic data, i.e., $\mu=\beta_k/\alpha_k$.
Fix $\kappa\in[0,1)$ and consider $t\ge t_\kappa$ where 
$
t_\kappa = \frac{1}{1+\mu}\ln \big(\frac{1}{1-\kappa}\big).
$
If the following inequality holds:
\begin{equation}\label{eq:mu-threshold}
    \mu \ge 
    \max \left\{
        \frac{1+ T\eta_{\FR}\kappa}{T\delta\kappa},
          \frac{\eta + \tfrac{2\eta_{\FR}}{\kappa}}{\delta}
    \right\},
\end{equation}
there exists $\lambda>0$ such that every solution of \eqref{eq:cts} satisfies:
\begin{equation}\label{eq:KL-contract-new}
d_{\FR}(\theta(t),\theta_e)
\le \frac{\pi}{2}\sqrt{e^{-\lambda (t-t_\kappa)}
D_{\KL}(\theta(t_\kappa)\|\theta_e)}
+ \upepsilon_{\FR},
\end{equation}
where:
\begin{equation}\label{eq:FR-eps}
\upepsilon_{\FR}
=
\frac{
\pi\,\eta\,\sqrt{n\,\max_i\theta_{e,i}}
}{
\sqrt{2}\,\kappa\,(\mu\delta-\eta)
\left(1-\dfrac{2\eta_{\FR}\max_i\theta_{e,i}}{\kappa(\mu\delta-\eta)}\right)
}.    
\end{equation}
\end{theorem}

Note that~\eqref{eq:KL-contract-new} implies that $\theta$ converges to a Fisher-Rao ball of size $\upepsilon_\FR$ for $t\to\infty$. We now compare the bound provided by~\eqref{eq:FR-eps} in Theorem~\ref{thm:main} with the bound obtained in~\cite{ gharesifard2025preventing} that establishes convergence to an Euclidean ball of size:
\begin{equation}\label{eq:E-upepsilon}
\upepsilon=\frac{\eta \kappa(\mu\delta-\eta) T}{(\mu+1)(\kappa(\mu\delta-\eta) T-1)}. 
\end{equation}
To more easily exhibit the relative scaling of the bounds, we assign scaling laws to $\delta$ and $\mu$ as functions of $n$, the dimension of the probability simplex $\Delta^n$. 

\begin{proposition}
Suppose that the assumptions of Theorem~\ref{thm:main}, and the ones in~\cite[Theorem~1]{ gharesifard2025preventing}, hold and that:
\[
\delta_n \sim n^{-\beta_0},\qquad \mu_n \sim c\,n^p, \qquad\|\varepsilon\|_\infty = \eta,
\]
for some
$\beta_0 > 1$, $c>0$, and
$p > 2\beta_0$.
Then: 
\begin{enumerate}
\item In the Euclidean case~\eqref{eq:E-upepsilon}: 
\[
\upepsilon  =  \Theta\bigl(n^{-p}\bigr).
\]
\item 
In the Fisher-Rao case~\eqref{eq:FR-eps}: 
    \[
    \upepsilon_{\FR}  
    =\Theta\bigl(n^{\,\frac{1}{2}+\beta_0-p}\bigr).
    \]
\end{enumerate}
\end{proposition}

\begin{proof}
We first prove~(1).  
Consider~\eqref{eq:E-upepsilon} given by: 
\[
\upepsilon
=\frac{\eta\kappa(\mu_n\delta_n-\eta)T}
{(\mu_n+1)\big(\kappa(\mu_n\delta_n-\eta)T-1\big)}, 
\]
with $\delta_n\sim n^{-\beta_0}$ and $\mu_n\sim c n^p$. 
Since $\mu_n\delta_n \sim c n^{p-\beta_0}$, with $ p\geq \beta_0 $, and $\mu_n+1\sim\mu_n$, we have that: 
\[
\upepsilon  \sim  \frac{\eta}{\mu_n}  =  \Theta(n^{-p}),
\]
which proves the claim.

To prove~(2), we first have to determine how $\theta_e$ scales with $ n $ under the assumptions. 
Manipulating~\eqref{eq:eq},
we obtain:
\[
\theta_e-\theta_0 = \frac{1}{1+\mu_n}\big(\tau(\theta_e)-\theta_0\big).
\]
Taking $\ell_\infty$ norms and noting $\tau(\theta_e),\theta_0\in\Delta^n$, we get:
\[
\|\theta_e-\theta_0\|_\infty  \le  \frac{2}{1+\mu_n}.
\]
Hence, if $\mu_n\sim c n^p$ with $p\ge 1$ we have:
\[
\|\theta_e-\theta_0\|_\infty = O(n^{-p}).
\]
    
    Noting that ${\max_i\theta_{0,i} \le 1- (n-1)\min_i\theta_{0,i}}$ and ${\min_i\theta_{0,i} = \delta_n\sim n^{-\beta_0}}$, we can establish that:
    \begin{align*}
        \max_i\theta_{e,i} &= \max_i\theta_{0,i} + O(n^{-p})\\
        &\le 1 - (n-1)\min_i\theta_{0,i} + O(n^{-p})\\
        &= 1 - (n-1)\Theta(n^{-\beta_0}) + O(n^{-p})\\
        &= 1-\Theta(n^{1-\beta_0}),\\[5pt]
         \min_i\theta_{e,i} &= \min_i\theta_{0,i} - O(n^{-p})\\
         &= \Theta(n^{-\beta_0}) - O(n^{-p})
         = \Theta(n^{-\beta_0}).
    \end{align*}
    
    We are now in a position to prove~(2), by considering~\eqref{eq:FR-eps}. 
    Recall that $\mu_n\sim c n^p$ and $\delta_n\sim n^{-\beta_0}$, so:
    \[
    \mu_n\delta_n \sim c n^{p-\beta_0}.
    \]
    Using $\max_i\theta_{e,i} \le 1-\Theta(n^{1-\beta_0})$ and $\min_i\theta_{e,i} = \Theta(n^{-\beta_0})$, we obtain:
    \begin{align*}
        \frac{2\eta_{\FR}\max_i\theta_{e,i}}
        {\kappa(\mu_n\delta_n-\eta)}
        &=
        \frac{\frac{2}{T}\left(\frac{\max_i\theta_{e,i}}{\min_i\theta_{e,i}}\right)}
        {\kappa(\mu_n\delta_n-\eta)}
        \le
        \frac{\left(\frac{1-\Theta(n^{1-\beta_0})}{\Theta(n^{-\beta_0})}\right)}
        {\Theta(n^{p-\beta_0})}\\
        &=
        \Theta\left(\frac{n^{\beta_0}}
        {n^{p-\beta_0}}\right)
        =
        \Theta\left(n^{2\beta_0 - p}\right),
    \end{align*}
    which tends to $0$ for every $p > 2\beta_0$.  Hence:
    \[
    1-\frac{2\eta_{\FR}\max_i\theta_{e,i}}
    {\kappa(\mu_n\delta_n-\eta)}
    = \Theta(1).
    \]
    Moreover, $\mu_n\delta_n-\eta \sim \mu_n\delta_n$, so the denominator of~\eqref{eq:FR-eps} is asymptotically proportional to $\mu_n\delta_n$, and $\upepsilon_{\FR}$ scales like:
    \[
    \frac{
    \pi\,\eta\,\sqrt{n\,\max_i\theta_{e,i}}
    }{
    \sqrt{2}\,\kappa\,\mu_n\delta_n
    }
    \le
    \frac{\sqrt{n\,(1-\Theta(n^{1-\beta_0}))}}
    {\Theta(n^{p-\beta_0})}
    =
    \Theta\left(n^{\frac{1}{2}+\beta_0-p}\right).
    \]
\end{proof}

Without structural information on  $\theta_e$ beyond the simplex constraints, a natural choice for the amount of human data required to obtain (for example) a $ O(1/n) $ decay is
$\mu_n=n^{\frac{5}{2}+\gamma}$, with $\gamma>0$. This keeps the  
Fisher-Rao error uniformly controlled as the dimension grows. By contrast, the Euclidean estimate does not account for the geometric cost of  
placing probability mass across $n$ coordinates: with $\mu_n\sim n$ it predicts  
an error of order $1/n$. The Fisher-Rao geometry analysis above reveals that this can be misleading, as $\mu_n\sim n$ does not prevent the Fisher-Rao error from growing, and in this sense, a stronger growth of $\mu_n$ is required.

We devote the next section to proving Theorem~\ref{thm:main}. 

\section{Proof of the main result}

We proceed in steps, proving some intermediate lemmas before the main result. First, we show that solutions of~\eqref{eq:cts} cannot approach the boundary asymptotically.

\begin{lemma}\label{lem:invariance}
Suppose that Assumptions~\ref{ass:eps_theta0}-\ref{ass:theta_e} hold, and let $\theta(\cdot)$ be the solution to~\eqref{eq:cts}. For any $\kappa\in[0,1)$, there exists a time:
\[
t_\kappa  =  \frac{1}{1+\mu}\ln \Big(\frac{1}{1-\kappa}\Big),
\]
such that, for all $t\ge t_\kappa$ and all $i\in\{1,2,\dots,n\}$:
\begin{equation}\label{eq:lower}
    \theta_i(t)  \ge  \frac{\kappa(\mu\delta-\eta)}{1+\mu}
     =:  \underline{\theta}  >  0,
\end{equation}
provided that $\mu\delta>\eta$.
\end{lemma}
\begin{proof}
For each $i$, we have that: 
\[
\dot\theta_i  =  -\theta_i + \mu(\theta_{0,i}-\theta_i) + \tau_i(\theta) + \varepsilon_i,
\]
where by assumption $\tau_i(\theta)\ge0$ and $\varepsilon_i\ge -\eta$.  
Hence:
\[
\dot\theta_i  \ge  -\theta_i + \mu(\theta_{0,i}-\theta_i) - \eta
 =  -(1+\mu)\theta_i + \mu\theta_{0,i} - \eta.
\]
Therefore:
\[
\theta_i(t)  \ge  e^{-(1+\mu)t}\theta_{i,0}
 + 
\big(1-e^{-(1+\mu)t}\big)\frac{\mu\theta_{0,i}-\eta}{1+\mu}.
\]
The first term decays exponentially, while the second term approaches the steady-state value
\(\tfrac{\mu\theta_{0,i}-\eta}{1+\mu}\).  
To ensure a uniform lower bound after some finite time, we select $t\ge t_\kappa$ such that:
\[
1 - e^{-(1+\mu)t}  \ge  \kappa,
\]
where $ \kappa \in [0,1) $, which is equivalent to enforcing that:
\[
t  \ge  \frac{1}{1+\mu}\ln \Big(\frac{1}{1-\kappa}\Big).
\]
As a result, for all $t\ge t_\kappa$, we then obtain:
\[
\theta_i(t)
 \ge 
\kappa\,\frac{\mu\theta_{0,i}-\eta}{1+\mu}.
\]
Finally, since by Assumption~\ref{ass:eps_theta0} each $\theta_{0,i}\ge \delta$, we have:
\[
\theta_i(t)
 \ge 
\frac{\kappa(\mu\delta-\eta)}{1+\mu}
 =: 
\underline{\theta},
\]
which is positive whenever $\mu\delta>\eta$, establishing \eqref{eq:lower}.
\end{proof}

The uniform floor $\underline\theta$ comes directly from the dynamics and is independent of metric; it is essential for our Fisher-Rao analysis, as the metric is not well-defined on the boundary set $ \partial \Delta^n$.
We now adopt the KL-divergence between $\theta(t)$ and the equilibrium $\theta_e$ as a Lyapunov function, and establish its Lie derivative along the vector field~\eqref{eq:cts}.
\begin{lemma}\label{lem:decomp}
Let: 
\[
V(\theta)=D_{\KL}(\theta\|\theta_e).
\]
Then, along any solution of \eqref{eq:cts} we have:
\begin{equation}\label{eq:Vdot}
\begin{gathered}
    \dot V(\theta(t))
    = \Big\langle\log\frac{\theta(t)}{\theta_e},\tau(\theta(t))-\tau(\theta_e)\Big\rangle\\
    - (1+\mu)\Big\langle\log\frac{\theta(t)}{\theta_e},\theta(t)-\theta_e\Big\rangle
    + \Big\langle\log\frac{\theta(t)}{\theta_e},\varepsilon(t)\Big\rangle .
\end{gathered}
\end{equation}
\end{lemma}

\begin{proof}
First note that $\nabla V(\theta)=\log(\theta/\theta_e)+\mathbf 1$, where $\mathbf 1$ is the vector whose all entries are $1$. For brevity, we now drop the dependency of the trajectories on $ t $. Since the trajectory stays on the simplex, $\langle \mathbf 1,\dot\theta\rangle=\frac{d}{dt}\sum_{i=1}^n\theta_i=0$, we have that:
\[
\dot V(\theta)=\langle \nabla V(\theta),\dot\theta\rangle
=\Big\langle\log\frac{\theta}{\theta_e},\dot\theta\Big\rangle.
\] 
Using now~\eqref{eq:cts}, 
adding and subtracting $\tau(\theta_e)$, we have that:
\begin{equation*}
\begin{gathered}
    \dot V(\theta)
    =\Big\langle\log\frac{\theta}{\theta_e},\tau(\theta)-\tau(\theta_e)\Big\rangle\\
    +\Big\langle\log\frac{\theta}{\theta_e},\tau(\theta_e)-\theta+\mu(\theta_0-\theta)\Big\rangle
    +\Big\langle\log\frac{\theta}{\theta_e},\varepsilon\Big\rangle .
\end{gathered}
\end{equation*}
We use the equilibrium condition $\tau(\theta_e)-\theta_e+\mu(\theta_0-\theta_e)=0$ to rewrite:
\begin{align*}
\tau(\theta_e)-\theta+\mu(\theta_0-\theta)
&= -\bigl[(\theta-\theta_e)+\mu(\theta-\theta_e)\bigr]\\
&= -(1+\mu)(\theta-\theta_e),
\end{align*}
which gives \eqref{eq:Vdot}.
\end{proof}

We are now finally fully equipped to prove the results stated in Theorem~\ref{thm:main}.

\begin{proof}[Proof of Main Theorem]
For $t\ge t_\kappa$, Lemma~\ref{lem:invariance} ensures $\theta_i(t)\ge \underline{\theta}$. 
From Lemma~\ref{lem:decomp}, along any trajectory we have:
\begin{equation*}
\begin{gathered}
    \dot V(\theta)
    =
    \underbrace{\left\langle\log\frac{\theta}{\theta_e},\tau(\theta)-\tau(\theta_e)\right\rangle}_{\text{(A)}}\\
    -(1+\mu)\underbrace{\Big\langle\log\frac{\theta}{\theta_e},\theta-\theta_e\Big\rangle}_{\text{(B)}}
    +\underbrace{\Big\langle\log\frac{\theta}{\theta_e},\varepsilon\Big\rangle}_{\text{(C)}}.
\end{gathered}
\end{equation*}

For (A), by Lemma~\ref{lem:softmax-FR}, $\tau$ is $\eta_{\FR}$–Lipschitz in the Fisher-Rao metric, so:
\[
\Big\langle\log\frac{\theta}{\theta_e},\tau(\theta)-\tau(\theta_e)\Big\rangle
\le \eta_{\FR}\|\log\theta-\log\theta_e\|_{\Diag(\theta_e)}^2.
\]
Since $ \theta_i(t) \geq \underline{\theta}$, the conditions of Lemma~\ref{lem:equiv} apply: 
\[
\|\log\theta-\log\theta_e\|_{\Diag(\theta_e)}^2 \leq a_1V(\theta), 
\]
where $a_1=\tfrac{1}{c_1}$, and therefore:
\[
\Big\langle\log\frac{\theta}{\theta_e},\tau(\theta)-\tau(\theta_e)\Big\rangle
\le \eta_{\FR}a_1V(\theta).
\]

For (B), we have that: 
\begin{align*}
\Big\langle \log\frac{\theta}{\theta_e},\theta-\theta_e\Big\rangle
&=\sum_i\theta_i\log\frac{\theta_i}{\theta_{e,i}}
+\sum_i\theta_{e,i}\log\frac{\theta_{e,i}}{\theta_i}\\
&=D_{\KL}(\theta\|\theta_e)+D_{\KL}(\theta_e\|\theta)\ge V(\theta),
\end{align*}
and therefore:
\begin{equation}\label{eq:sym-main}
-(1+\mu)\Big\langle \log\frac{\theta}{\theta_e},\theta-\theta_e\Big\rangle \leq -(1+\mu)V(\theta).    
\end{equation}

For (C), by H\"older’s inequality, we have that:
\[
\Big|\Big\langle\log\frac{\theta}{\theta_e},\varepsilon\Big\rangle\Big|
\le \Big\|\log\frac{\theta}{\theta_e}\Big\|_1\|\varepsilon\|_\infty.
\]
Letting $x=\log\theta-\log\theta_e$ and using Cauchy-Schwarz:
\begin{align*}
    \Big\|\log\frac{\theta}{\theta_e}\Big\|_1=
    \|x\|_1
    &=\sum_i|x_i|\\
    &\le
    \Big(\sum_i\frac{1}{\theta_{e,i}}\Big)^{1/2}
    \Big(\sum_i\theta_{e,i}x_i^2\Big)^{1/2}.
\end{align*}
Note that by definition, $
\sum_i\theta_{e,i}x_i^2=\|\log\theta-\log\theta_e\|_{\Diag(\theta_e)}^2 $. Hence, using Lemma~\ref{lem:equiv} again, we have that:
\[
\sum_i\theta_{e,i}x_i^2  \le  \frac{1}{c_1}V(\theta).
\]
Whenever $\theta_i\ge\underline{\theta}$, we obtain:
\[
\Big\|\log\frac{\theta}{\theta_e}\Big\|_1
\le
\tilde{a}_2\sqrt{V(\theta)},
\]
where $\tilde{a}_2=\sqrt{\tfrac{1}{c_1}\sum_i1/\theta_{e,i}}$. Note that since $ 1/\theta_{e,i} \leq 1/\underline{\theta}$, we have that $ \sum_i1/\theta_{e,i} \leq n/\underline{\theta}$ and hence:
\[
\tilde{a}_2 \le a_2:=\sqrt{\tfrac{n}{c_1\underline{\theta}}}.
\]
Substituting these bounds into the expression for $\dot V$ yields:
\[
\dot V(\theta)
\le (\eta_{\FR}a_1-(1+\mu))V(\theta)+a_2\eta\sqrt{V(\theta)},
\quad
\eta:=\|\varepsilon\|_\infty.
\]

Next we verify that the lower bound on $\mu$ in our assumption guarantees that $(1+\mu)-\eta_{\FR}a_1>0$.  
Recall that
$
\underline{\theta}=\frac{\kappa(\mu\delta-\eta)}{1+\mu}
$,
and that
$
c_1 = \frac{\underline{\theta}}{2\max_i\theta_{e,i}}
$ by Lemma~\ref{lem:equiv}.
Therefore:
\[
a_1 = \frac{1}{c_1}
= \frac{2\max_i\theta_{e,i}}{\underline{\theta}}
= \frac{2\max_i\theta_{e,i}(1+\mu)}{\kappa(\mu\delta-\eta)}\le 
\frac{2(1+\mu)}{\kappa(\mu\delta-\eta)},
\]
where we have used the fact that $\max_i\theta_{e,i}\le 1$. 
Thus:
\[
(1+\mu)-\eta_{\FR}a_1
 \ge 
(1+\mu)\left(1-\frac{2\eta_{\FR}}{\kappa(\mu\delta-\eta)}\right),
\]
and to ensure the right-hand side is positive, we require:
\[
1-\frac{2\eta_{\FR}}{\kappa(\mu\delta-\eta)} > 0
\quad\Longleftrightarrow\quad
\mu\delta-\eta > \frac{2\eta_{\FR}}{\kappa},
\]
but this is exactly enforced by the assumption that
${
\mu  \ge  \frac{\eta+\tfrac{2\eta_{\FR}}{\kappa}}{\delta}
}$.
Given this, we let:
\[
0<\lambda<(1+\mu)-\eta_{\FR}a_1
\]
and apply Young's inequality $ab\le \lambda a^2+\tfrac{b^2}{4\lambda}$ with $a=\sqrt{V(\theta)}$ and $b=a_2\eta$ to obtain: 
\[
\dot V(\theta)\le -\big((1+\mu)-\eta_{\FR}a_1-\lambda\big)V(\theta)+\frac{a_2^2\eta^2}{4\lambda}.
\]
We can indeed choose:
\[
\lambda=\tfrac12\Big((1+\mu)-\frac{\eta_{\FR}}{c_1}\Big),
\]
so that $(1+\mu)-\eta_{\FR}a_1-\lambda=\lambda$, and therefore:
\[
\dot V(\theta)\le -\lambda V(\theta)+\frac{a_2^2\eta^2}{4\lambda}.
\]
Consequently, for all $ t\ge t_\kappa$ we have that: 
\[
V(\theta(t))
\le e^{-\lambda(t-t_\kappa)}V(\theta(t_\kappa))
+\overline\upepsilon_{\FR}^2\bigl(1-e^{-\lambda(t-t_\kappa)}\bigr),
\]
where:
$
\overline\upepsilon_{\FR}^2:=\frac{a_2^2\eta^2}{4\lambda^2}
$.
By substituting $a_2$, $\lambda$, and $\underline{\theta}$, we have: 
\begin{equation}\label{eq:auxFR}
\overline\upepsilon_{\FR}^2
=\frac{n\eta^2(1+\mu)}{c_1\kappa(\mu\delta-\eta)\big((1+\mu)-\tfrac{\eta_{\FR}}{c_1}\big)^2}.    
\end{equation}
Using: 
\[
c_1=\frac{\underline{\theta}}{2\max_i\theta_{e,i}}
    =\frac{\kappa(\mu\delta-\eta)}{2(1+\mu)\max_i\theta_{e,i}},
\]
we obtain one of the terms in the denominator of~\eqref{eq:auxFR}:
\[
c_1\kappa(\mu\delta-\eta)
=\frac{\kappa^{2}(\mu\delta-\eta)^{2}}{2(1+\mu)\max_i\theta_{e,i}}. 
\]
To compute the other term, we write:
\begin{align*}
    (1+\mu)-\frac{\eta_{\FR}}{c_1}
    &=(1+\mu)-\frac{2(1+\mu)\eta_{\FR}\max_i\theta_{e,i}}{\kappa(\mu\delta-\eta)}\\
    &=(1+\mu) \left(1-\frac{2\eta_{\FR}\max_i\theta_{e,i}}{\kappa(\mu\delta-\eta)}\right),
\end{align*}
and substituting both identities yields:
\[
\overline\upepsilon_{\FR}^2
=
\frac{
2n\,\eta^{2}\,\max_i\theta_{e,i}
}{
\kappa^{2}(\mu\delta-\eta)^{2}
\left(1-\dfrac{2\eta_{\FR}\max_i\theta_{e,i}}{\kappa(\mu\delta-\eta)}\right)^{2}
}.
\]

Finally, by Lemma~\ref{lemma:dFR-D}, $d_{\FR}(\theta,\theta_e)\le \frac{\pi}{2}\sqrt{V(\theta)}$, and therefore:
\begin{align*}
d_{\FR}(\theta(t),\theta_e)
&\le \frac{\pi}{2}\sqrt{e^{-\lambda(t-t_\kappa)}D_{\KL}(\theta(t_\kappa)\|\theta_e)
+\overline\upepsilon_{\FR}^2}\\
&\le \frac{\pi}{2}\sqrt{e^{-\lambda(t-t_\kappa)}D_{\KL}(\theta(t_\kappa)\|\theta_e)}
+\upepsilon_{\FR},
\end{align*}
with $\upepsilon_\FR$ matching~\eqref{eq:FR-eps}, which establishes \eqref{eq:KL-contract-new}.

\end{proof}

\section{Conclusion}
In this work, we tackled the fundamental challenge of bounding the amount of human data required to prevent model collapse. By modeling the iterative training of generative models as a closed-loop stochastic process, we demonstrated that trajectories converge to a stable ball in the Fisher-Rao metric when the human-to-synthetic data ratio exceeds a specific threshold. We fundamentally departed from prior analysis that relies on the standard Euclidean metric. Ultimately, leveraging the information-geometric structure of the probability simplex naturally accounts for the distortions introduced by working on the simplex, yielding bounds that scale properly with the dimension of the underlying categorical distributions.

\bibliographystyle{ieeetr}%
\bibliography{biblio}

\appendix
Following we state and prove a number of technical Lemmas that the main result relies on.
\section{Auxiliary Lemmas}
\begin{lemma}\label{lem:FR-Hellinger}
For any $\theta,\vartheta\in\Delta^n$, the following inequalities hold:
\[
\sqrt{2}H(\theta,\vartheta)\le d_{\FR}(\theta,\vartheta)
\le \tfrac{\pi}{\sqrt{2}}H(\theta,\vartheta).
\]
\end{lemma}

\begin{proof}
Let $c:=\big\langle \sqrt{\theta},\sqrt{\vartheta}\big\rangle \in [0,1]$. By definition, $d_{\FR}(\theta,\vartheta)=\arccos(c)$ and $H^2(\theta,\vartheta)=1-c$. Using the identity $1-\cos x=2\sin^2(x/2)$, we have $H(\theta,\vartheta)=\sqrt{2}\sin \big(\tfrac12 d_{\FR}(\theta,\vartheta)\big)$. Since $\frac{2}{\pi} y \le \sin y \le y$ for $y\in[0,\tfrac{\pi}{2}]$, substituting $y = d_{\FR}/2$ yields the stated bounds.
\end{proof}

\begin{lemma}\label{lemma:dFR-D}
For any $\theta,\vartheta\in\Delta^n_{\mathrm{int}}$, we have: $$d_{\FR}(\theta,\vartheta)\le\frac{\pi}{2}\sqrt{D_{\mathrm{KL}}(\theta\|\vartheta)}.$$
\end{lemma}
\begin{proof}
By \cite[Lemma 2.4]{tsybakov2009introduction}, the KL divergence bounds the squared Hellinger distance as $D_{\mathrm{KL}}(\theta\|\vartheta)\ge 2H^2(\theta,\vartheta)$. Combining this with the upper bound in Lemma \ref{lem:FR-Hellinger} immediately establishes the result.
\end{proof}

\begin{lemma}\label{lemma:app1}
For all $u>0$ we have:
\[
u\log u-(u-1)\ge(1-\sqrt{u})^2.\]
\end{lemma}
\begin{proof}
 Set $s=\sqrt{u}$ and define $g(s):=2s^2\log s - s^2 + 2s - 2$.
We write $g$ as $g(a)=2s\left(s
\log s-s+1\right)$. Since $2s> 0$ for $s> 0$, it suffices to show that $h(s)=s
\log s-s+1> 0$ for $s> 0$. By Taylor's theorem $h(s)=h(1)+h'(1)(s-1)+\frac{1}{2}h''(s')(s-1)^2$  for some $s'> 0$. Computing the derivatives we obtain $h(s)=1+0(s-1)+\frac{1}{2}\frac{1}{s'}(s-1)^2$. By inspection we see that $h(s)>0$.
\end{proof}

\begin{lemma}\label{lem:scalar-g}
Let $ \map{g}{\real_{>0}}{\real}$ be defined as:
\[
g(r) := r\log r - (r-1).
\]
Then, for all $r>0$:
\begin{equation}\label{eq:scalar-g-bounds}
\frac12\min\{r,1\}(\log r)^2
 \le 
g(r)
 \le 
\frac12\max\{r,1\}(\log r)^2.
\end{equation}
\end{lemma}

\begin{proof}
We begin by computing the derivative of $g$:
\begin{equation}\label{eq:g-der}
g'(r) = \log r,\qquad r>0.    
\end{equation}

We define the functions:
\[
\phi(r) := 2g(r) - (\log r)^2,
\qquad
\psi(r) := r(\log r)^2 - 2g(r).
\]
Since $g(1)=0$, we have $\phi(1)=\psi(1)=0$. We now prove the result by considering two cases:

\emph{Case $r\ge 1$:}
For the lower bound, using~\eqref{eq:g-der}:
\[
\phi'(r)
= 2g'(r) - \frac{2\log r}{r}
= 2\log r\Bigl(1-\frac1r\Bigr).
\]
If $r\ge1$, then $\log r\ge0$ and $1-\tfrac1r\ge0$, so $\phi'(r)\ge0$ on $[1,\infty)$.
Thus $\phi(r)\ge\phi(1)=0$ for all $r\ge1$, which implies:
\[
g(r) \ge \frac{1}{2}(\log r)^2,\qquad r\ge1.
\]

For the upper bound, we use $\psi$:
\begin{gather*}
    \psi'(r)
    = \frac{d}{dr}\bigl[r(\log r)^2\bigr] - 2g'(r)\\
    = (\log r)^2 + 2\log r - 2\log r
    = (\log r)^2 \ge 0.
\end{gather*}
Hence $\psi$ is increasing on $(0,\infty)$; in particular, $\psi(r)\ge\psi(1)=0$ for $r\ge1$. Therefore
\[
g(r) \le \frac{1}{2}\,r(\log r)^2,\qquad r\ge1.
\]
This establishes the bounds for the case where $ r \geq 1$. 

\emph{Case $0<r\le 1$:}
For the lower bound, we naturally use $\psi$ this time. 
As above, $\psi'(r)=(\log r)^2\ge0$ for all $r>0$, so $\psi$ is increasing on $(0,\infty)$ and:
\[
\psi(r)\le\psi(1)=0,\qquad 0<r\le1.
\]
Thus:
\[
r(\log r)^2 - 2g(r) \le 0
\quad\Longrightarrow\quad
g(r) \ge \frac12\,r(\log r)^2,\quad 0<r\le1.
\]

For the upper bound, we use $\phi$ again.
From:
\[
\phi'(r)=2\log r\Bigl(1-\frac1r\Bigr),
\]
we see that for $0<r\le1$ we have $\log r\le0$ and $1-\frac1r\le0$, so $\phi'(r)\ge0$ on $(0,1]$.
Hence $\phi$ is increasing on $(0,1]$ and:
\[
\phi(r)\le\phi(1)=0,\qquad 0<r\le1.
\]
Therefore:
\[
2g(r) - (\log r)^2 \le 0
\quad\Longrightarrow\quad
g(r) \le \frac12(\log r)^2,\quad 0<r\le1.
\]
This establishes the bound for the case where $0<r\le1 $.
Combining the two cases yields \eqref{eq:scalar-g-bounds} for all $r>0$.
\end{proof}

\begin{lemma}\label{lem:softmax-FR}
Let $\theta_e\in\Delta^n$, $T>0$, $\minthetae:=\min_i\theta_{e,i}>0$, and: 
\[
\eta_{\FR}\ :=\ \frac{1}{T\,\minthetae}.
\]
Then, for all $\theta\in\Delta^n_{\mathrm{int}}$ the following inequality holds:
\begin{equation}\label{eq:frlip}
\big\langle \log\theta-\log\theta_e,\ \tau(\theta)-\tau(\theta_e)\big\rangle
\ \le\ 
\eta_{\FR}\,\|\log\theta-\log\theta_e\|^2_{\mathrm{diag}(\theta_e)}.
\end{equation}
\end{lemma}

\begin{proof}
Let $\varphi=\log\theta$ and $\varphi_e=\log\theta_e$ and consider:
\[
f(\varphi)\ :=\ T\log \Big(\sum_{i=1}^n e^{\varphi_i/T}\Big).
\]
Clearly, $\nabla f(\varphi)=\tau(\theta)$. For $i,j\in\{1,\dots,n\}$, we have that: 
\[
\frac{\partial (\nabla f)_i}{\partial \varphi_j}
=\frac{\partial \tau_i(\theta)}{\partial \varphi_j}
=\frac{1}{T}\big(\tau_i(\theta)\,\delta_{ij}-\tau_i(\theta)\tau_j(\theta)\big), 
\]
where $\delta_{ij}=1$ when $i=j$ and zero otherwise, which follows from taking the derivative of: 
\[
\tau_i(\theta)=\tau_i(\varphi)=\frac{e^{\varphi_i/T}}{\sum_{k=1}^ne^{\varphi_k/T}}. 
\]
Therefore:
\[
\nabla^2 f(\varphi)\ =\ \frac{1}{T}\Big(\mathrm{diag}(\tau(\theta))-\tau(\theta)\tau(\theta)^\top\Big).
\]
For simplicity of calculations to follow, we let \(p:=\tau(\theta)\). We also let \(v\in\mathbb{R}^n\) be arbitrary, with $\|v\|_2=1.$ Then:
\begin{gather*}
v^\top \Big(\mathrm{diag}(p)-pp^\top\Big)v
= \sum_i p_i v_i^2 - \Big(\sum_i p_i v_i\Big)^2\\ \le\ \sum_i p_i v_i^2 \leq \|v\|_2^2,
\end{gather*}
where we have used the fact that $\sum_i p_i=1$. 
Hence, taking the supremum over unit vectors:
\begin{equation}
\label{Lipschitz}
\big\|\nabla^2 f(\varphi)\big\|_2
=\frac{1}{T}\sup_{\|v\|_2=1} v^\top \Big(\mathrm{diag}(p)-pp^\top\Big)v
\ \le\ \frac{1}{T}.
\end{equation}
Therefore,
for all $ x,y \in \real^n$ we have that:
\[
\langle x-y,\ \nabla f(x)-\nabla f(y)\rangle\ \le\ \frac{1}{T}\,\|x-y\|_2^2,
\]
where we used~\eqref{Lipschitz} as an upper bound for the Lipschitz constant of $\nabla f$.
Applying this with $x=\varphi$ and $y=\varphi_e$ gives:
\[
\langle \varphi-\varphi_e,\ \tau(\theta)-\tau(\theta_e)\rangle
\ \le\ \frac{1}{T}\,\|\varphi-\varphi_e\|_2^2.
\]
Since $\mathrm{diag}(\theta_e)\succeq \minthetae I$, we have
$\|\varphi-\varphi_e\|_2^2 \le (\minthetae)^{-1}\|\varphi-\varphi_e\|^2_{\mathrm{diag}(\theta_e)}$, and therefore, a substitution, yields~\eqref{eq:frlip}.
\end{proof}

\begin{lemma}\label{lem:equiv}
Fix $\underline{\theta}\in(0,1)$ and let:
\[
\mathcal{I}:=\Big\{\theta\in\Delta^n \mid \theta_i\ge \underline{\theta}\ \mathrm{for \ all } \  i\Big\}.
\]
Then, for all $\theta\in\mathcal{I}$:
\begin{equation}\label{eq:equiv}
\begin{gathered}
c_1 \|\log\theta-\log\theta_e\|_{\Diag(\theta_e)}^2
\le
D_{\mathrm{KL}}(\theta\|\theta_e)\\
\le
c_2 \|\log\theta-\log\theta_e\|_{\Diag(\theta_e)}^2,
\end{gathered}
\end{equation}
with:
\[
c_1 = \frac{\underline{\theta}}{2\max_i\theta_{e,i}},
\qquad
c_2 = \frac{1}{2\min_i\theta_{e,i}}.
\]
\end{lemma}

\begin{proof}
Let us define: 
\[
r_i:=\frac{\theta_i}{\theta_{e,i}}\in(0,\infty),
\qquad
x_i:=\log r_i=\log\theta_i-\log\theta_{e,i}.
\]
Since $\sum_i\theta_i=\sum_i\theta_{e,i}=1$, we have $\sum_i\theta_{e,i}(r_i-1)=0$. To proceed, observe that the KL divergence can be written as a
\emph{Bregman divergence}. Recall that for a convex function 
$f:\mathbb{R}^+\to\mathbb{R}$, the Bregman divergence between $u,v>0$ is:
\[
D_f(u\|v) = f(u)-f(v)-f'(v)(u-v).
\]
Consider the convex function:
\[
f(u)=u\log u,\qquad f'(u)=\log u+1,\qquad f''(u)=\frac1u.
\]
For each $i$:
\begin{equation*}
\begin{aligned}
    D_f(r_i\|1)
    &= f(r_i)-f(1)-f'(1)(r_i-1)\\
    &= r_i\log r_i - (r_i-1)\\
    &=: g(r_i),
\end{aligned}
\end{equation*}
since $f(1)=0$ and $f'(1)=1$.  Therefore:
\[
D_{\mathrm{KL}}(\theta\|\theta_e)
= \sum_{i=1}^n \theta_{e,i} g(r_i),
\]
as we previously established that $\sum_i\theta_{e,i}(r_i-1)=0$.
By Lemma~\ref{lem:scalar-g} in the Appendix, for every $r>0$:
\[
\frac12\min\{r,1\}(\log r)^2
 \le 
g(r)
 \le 
\frac12\max\{r,1\}(\log r)^2.
\]
Applying it with $r=r_i$ and noting $x_i=\log r_i$, we obtain:
\begin{equation}\label{eq:sandwich}
\begin{aligned}
\frac12\sum_{i=1}^n \theta_{e,i}\min\{r_i,1\}x_i^2
 &\le 
D_{\mathrm{KL}}(\theta\|\theta_e)\\
 &\le 
\frac12\sum_{i=1}^n \theta_{e,i}\max\{r_i,1\}x_i^2.
\end{aligned}
\end{equation}

We now produce uniform bounds on the factors $\min\{r_i,1\}$ and $\max\{r_i,1\}$. Because $\theta\in\mathcal{I}$ and $\theta\in\Delta^n$:
\begin{equation*}
\begin{gathered}
    \underline\theta \le \theta_i \le 1,\qquad
    0<\theta_{e,i}\le \bar\theta_e:=\max_j\theta_{e,j},\\
    r_i=\frac{\theta_i}{\theta_{e,i}}\in\Big[\frac{\underline\theta}{\theta_{e,i}},\frac{1}{\theta_{e,i}}\Big].
\end{gathered}
\end{equation*}
Hence:
\begin{equation*}
\begin{gathered}
    \min\{r_i,1\}\ge \min\Big\{\frac{\underline\theta}{\theta_{e,i}},1\Big\}\ge\min\Big\{\frac{\underline\theta}{\bar\theta_e},1\Big\},\\
    \max\{r_i,1\}\le\max\Big\{\frac{1}{\theta_{e,i}},1\Big\}=\frac{1}{\theta_{e,i}},
\end{gathered}
\end{equation*}
where for the last equality we use the fact that $\theta_{e,i}\le 1$. Substituting these bounds into \eqref{eq:sandwich} gives:
\begin{equation}\label{eq:preKL-bounds}
\frac12\min \Big\{\frac{\underline\theta}{\bar\theta_e},1\Big\}\sum_{i=1}^n \theta_{e,i}x_i^2
\le
D_{\mathrm{KL}}(\theta\|\theta_e)
\le
\frac12\sum_{i=1}^n x_i^2.
\end{equation}
Finally, since
$\theta_{e,i}\le \bar\theta_e$ and $\theta_{e,i}\ge \underline\theta_e:=\min_i\theta_{e,i}$, we have:
\begin{equation}\label{eq:xnorm-weights}
\sum_{i=1}^n x_i^2 \le \frac{1}{\underline\theta_e}\sum_{i=1}^n \theta_{e,i}x_i^2.
\end{equation}
Combining \eqref{eq:preKL-bounds} and \eqref{eq:xnorm-weights} 
yields:
\[
\frac12\min \Big\{\frac{\underline\theta}{\bar\theta_e},1\Big\}
\sum_{i=1}^n \theta_{e,i}x_i^2
 \le 
D_{\mathrm{KL}}(\theta\|\theta_e)
 \le 
\frac{1}{2\underline\theta_e}
\sum_{i=1}^n \theta_{e,i}x_i^2,
\]
which is exactly:
\begin{equation*}
\begin{gathered}
    \frac12\min \Big\{\frac{\underline\theta}{\bar\theta_e},1\Big\}
    \|\log\theta-\log\theta_e\|_{\Diag(\theta_e)}^2
    \le
    D_{\mathrm{KL}}(\theta\|\theta_e)\\
    \le
    \frac{1}{2\underline\theta_e}
    \|\log\theta-\log\theta_e\|_{\Diag(\theta_e)}^2.
\end{gathered}
\end{equation*}
Note that $\bar\theta_e = \max_i \theta_{e,i}\ge 1/n$ while $\underline\theta\le 1/n$, so $\underline\theta/\bar\theta_e\le 1$ and therefore
$
\min\Big\{\frac{\underline\theta}{\bar\theta_e},1\Big\}
=
\frac{\underline\theta}{\bar\theta_e}
$,
obtaining the lower constant 
$c_1=\underline\theta/(2\bar\theta_e)$.
This establishes \eqref{eq:equiv} with
$
c_1=\underline\theta/(2\bar\theta_e)
$
and
$
c_2=1/(2\underline\theta_e).
$
\end{proof}

\end{document}